\documentclass[wcp]{jmlr}
\usepackage{microtype}
\usepackage{hyperref}
\usepackage{booktabs}
\usepackage{amsfonts}
\usepackage{xcolor}
\usepackage{multirow}
\definecolor{darkblue}{rgb}{0,0,0.5}
\hypersetup{colorlinks=true,citecolor=darkblue,linkcolor=darkblue,urlcolor=darkblue}
\makeatletter
\let \@jmlrpages \@empty
\makeatother
\newcommand{\ckad}{\mathrm{CKA}_{\Delta}}
\newcommand{\ckar}{\mathrm{CKA}_{+}}
\jmlrvolume{353,}
\jmlryear{2026}
\jmlrworkshop{ACML 2026}
\title[Concept-Specific Structural Alignment]{Contrastive-Difference CKA Reveals Concept-Specific Structural Alignment Across Language Model Architectures}
\author{\Name{Xueping Gao} \Email{xueping.gxp@alibaba-inc.com}\\\addr{Alibaba Cloud Computing, Hangzhou, China}}
\editors{Andy Song, Bo Han and Sarah Erfani}
\begin{document}
\makeatletter\let\@jmlrpages\@empty\makeatother
\maketitle

\begin{abstract}
Do language models encode the same prompted concepts in structurally compatible
ways, beyond generic similarity induced by shared questions? We study
contrastive-difference centered kernel alignment ($\ckad$), which compares
per-question activation differences between opposing concept instructions.
A matched same-concept versus cross-concept comparison separates the question
of concept-conditioned structure from the magnitude of representational similarity.
Across four checkpoints and eight personality-related traits, a fixed-final-layer
analysis gives mean same- and cross-trait $\ckad$ of 0.7259 and 0.6828.
Trait-label randomization detects correspondence conditional on this model and
prompt set ($p=1/10001$, one-sided), but a model-and-trait-weighted 95\% interval
for the difference is $[-0.0330,0.1017]$, limiting broader generalization.
Strict question-group holdouts retain high same-trait linear recoverability.
More stringently, maps and preprocessing fitted on seven traits transfer a
source polarity classifier for the eighth trait to handwritten test questions:
mean affine and orthogonal accuracies are 88.10\% and 88.57\%, versus 47.70\%
for shuffled pairing. Performance varies substantially across traits and model
directions. Synthetic controls show why moderate CKA and high affine accuracy
alone do not establish functional equivalence. The results support a scoped
diagnostic of concept-specific structure and held-out linear recoverability;
they do not establish universal invariance or a robust advantage over raw CKA
across architectures.
\end{abstract}

\section{Introduction}
Comparing representations across language models can help determine whether a
concept-sensitive probe has a meaningful counterpart in another model.
However, a large similarity score can arise from the same question content,
instruction format, or broad polarity cues rather than correspondence of the
specific concept of interest. A second ambiguity arises when a flexible map
recovers a binary label almost perfectly: success on that label need not imply
equivalence of the rest of the representation or of the models' behavior.

We investigate these ambiguities using paired positive and negative instructions
for eight personality-related traits. For a fixed question, the difference of
the two activation vectors provides a concept-conditioned representation.
We compare the matrices of these differences with centered kernel alignment
(CKA; \citealp{kornblith2019similarity}). Our diagnostic is the contrast between
same-trait and cross-trait comparisons on an identical model and question grid.
The high raw score itself is not the evidence of specificity. Differencing can
cancel an additive shared-content component, but nonlinear content--instruction
interactions can remain; specificity must therefore be tested empirically.

This distinction also guides our transfer experiments. Earlier work establishes
that learned transformations can transfer concept representations and linear
features between language models \citep{huang2025crossmodel,chen2025stitching}.
We do not claim that finding a useful affine map is itself new. Our contribution
is to connect paired contrastive differencing to matched same-versus-cross
comparisons and explicit controls, while separately testing which forms of
linear recovery survive stronger holdouts.

The main analysis covers four specific instruction-tuned checkpoints and eight
traits, using fixed layers rather than selecting the strongest layer per trait.
We treat overlapping model pairs and traits as dependent. We then distinguish
three levels of transfer evaluation: held-out question groups for a same-trait
map; cross-trait control maps; and one pooled map whose preprocessing and
alignment training exclude the evaluation trait entirely. The last experiment
still trains a source classifier using labeled examples of the evaluation
trait: it tests an unseen trait for the coordinate map, not zero-shot
classification of a concept never labeled anywhere.

The results support a limited but useful contribution. The observed
same-minus-cross $\ckad$ effect is positive, and conditional trait-label
randomization detects correspondence. The wider model/trait interval includes
zero. Concept-held-out maps achieve substantially above-chance mean accuracy,
but four of the 96 affine direction--trait cells are at chance. We report these
boundaries alongside the positive results, including the uncertainty in the
advantage over matched raw CKA.

\section{Related Work}
\paragraph{Representational similarity and its interpretation.}
CKA compares representational geometry without requiring equal feature
dimensions \citep{kornblith2019similarity}. The invariances of a metric are not
sufficient to establish what it measures in a particular application.
\citet{ding2021grounding} formulate sensitivity to functionally relevant changes
and specificity to irrelevant changes as empirical tests of similarity measures.
Our calibration follows this perspective: an orthogonal shared representation
provides a positive control, anisotropic scaling tests a mismatch between CKA
and recoverability, and independent content with shared labels tests whether
high transfer can occur without shared content geometry. These are deliberately
limited controls, not a complete certification of metric validity.

\paragraph{Cross-model concept and feature transfer.}
\citet{huang2025crossmodel} study transformations of concept representations
between language models. \citet{chen2025stitching} use affine model stitching
to transfer probes, steering vectors, and sparse-autoencoder features. Their
results motivate testing whether a map generalizes beyond its fitting concepts.
Our main novelty is the concept-conditioned same-versus-cross diagnostic,
not linear mapping alone. We test pooled maps on a held-out trait and present
classification recovery as an operational measurement, without equating it
with causal steering transfer or behavioral equivalence.

\paragraph{From prompted traits to reasoning trajectories.}
Our observations are single-turn, prompt-conditioned residual activations.
Agent systems and trajectory visualizations introduce additional units of
comparison. AlphaApollo organizes multi-turn reasoning and verification
\citep{zhou2026alphaapollo}; Landscape of Thoughts constructs features for
intermediate reasoning states \citep{zhou2025landscape}. Section~\ref{sec:agents}
specifies how these settings could test the limits of the present diagnostic.

\section{Method}
\subsection{Paired contrasts and matched comparisons}
For checkpoint $m$, trait $t$, question $i$ and fixed layer $\ell$, let
$a^{+}_{mti\ell},a^{-}_{mti\ell}\in\mathbb R^{d_m}$ be residual-stream
activations at the last prompt token under opposing trait instructions.
Stacking the $n$ paired differences gives
\[
 D_{mt\ell}[i,:]=a^{+}_{mti\ell}-a^{-}_{mti\ell}.
\]
We use three terms consistently: \emph{contrastive-difference CKA} is CKA
computed on these paired differences; \emph{geometric alignment} is the
descriptive CKA score; and \emph{linear recoverability} is held-out polarity
classification after fitting a coordinate map on training data.

For linear CKA, form centered Gram matrices $K=HDD^\top H$, with
$H=I-\mathbf1\mathbf1^\top/n$. The computation uses an unbiased-HSIC-style
normalization \citep{song2012feature}: setting Gram diagonals to zero, define
\begin{equation}
 U(K,L)=\frac{\operatorname{tr}(KL)
 +\frac{(\mathbf1^\top K\mathbf1)(\mathbf1^\top L\mathbf1)}{(n-1)(n-2)}
 -\frac{2\mathbf1^\top KL\mathbf1}{n-2}}{n(n-3)}.
 \label{eq:hsic}
\end{equation}
CKA is $U(K,L)/\sqrt{U(K,K)U(L,L)}$ when the denominator is nonzero.
As numerical safeguards, the implementation takes the absolute self-product
inside the square root, returns zero when the denominator is at most $10^{-10}$,
and clips the ratio to $[-1,1]$. The ratio itself is not an unbiased estimator.
The fixed-layer CKA analysis uses full-width activations; PCA-50 is used in
the transfer experiments, not in these primary CKA calculations.

Let $G_{abtu}=\operatorname{CKA}(D_{at},D_{bu})$ at one fixed depth, where
$a<b$ selects one of six unordered model pairs and $t,u$ range over eight traits.
The effect of interest is
\begin{equation}
 E_\Delta=\frac1{48}\sum_{a<b,t}G_{abtt}
 -\frac1{336}\sum_{a<b,t\ne u}G_{abtu}.
 \label{eq:effect}
\end{equation}
The matched raw-positive control replaces $D_{mt}$ with $A^+_{mt}$ on the
same questions, model pairs, traits and layers. We also evaluate the paired
effect gain $E_\Delta-E_+$. These contrasts ask different questions from the
absolute similarity of two representations.

\subsection{What differencing can and cannot guarantee}
Consider a stylized additive model
$a^+=s+p+\epsilon$ and $a^-=s-p+\epsilon'$. Assume these random components
are centered and mutually uncorrelated, and the two noise terms have equal
covariance and zero cross-covariance. Write
$S=\operatorname{tr}\operatorname{Cov}(s)$,
$P=\operatorname{tr}\operatorname{Cov}(p)$, and
$N=\operatorname{tr}\operatorname{Cov}(\epsilon)$.
Thus $N=d\sigma^2$ if noise is isotropic with covariance $\sigma^2I$.

\begin{proposition}[Signal fraction in the additive model]
\label{prop:variance}
If $S+P+N>0$ and $4P+2N>0$, the signal fractions
\[
 R_+=\frac{P}{S+P+N},\qquad R_\Delta=\frac{4P}{4P+2N}
\]
satisfy $R_\Delta\geq R_+$. For $P>0$, equality holds exactly when
$S=N=0$. For $P=0$ and valid denominators, both fractions are zero.
\end{proposition}
\begin{proof}
The shared component cancels in $a^+-a^-=2p+\epsilon-\epsilon'$.
The covariance traces of the raw and differenced vectors are $S+P+N$ and
$4P+2N$. Subtracting the two fractions gives
\[
 R_\Delta-R_+=\frac{4PS+2PN}{(4P+2N)(S+P+N)}\geq0.
\]
The stated equality cases follow because the denominator is positive.
\end{proof}

This is a variance-fraction statement, not a theorem that CKA, a ratio of
cross-model dependence estimates, becomes more discriminative. If $P=N=0$,
the differenced representation has zero variance and its signal fraction is
undefined; it must not be assigned a successful diagnostic value. Correlated
pole noise, imperfectly shared content and concept--content interactions also
fall outside the proposition's assumptions.

The original unconditional discrimination proposition is withdrawn. A scalar
proxy involving an orthogonality-slack parameter $0\leq\eta<1$ can yield
\[
 F=\frac{4(1-\eta)(S+P+N)}{4P(1+\eta)+2N},
\]
but even this proxy exceeds one only if
$4(1-\eta)S+(2-4\eta)N-8\eta P>0$. It is not a general expression for
$E_\Delta/E_+$. We use empirical matched controls, rather than this proxy,
to assess discrimination. No RBF-kernel guarantee follows from the proposition.

\subsection{Transfer without evaluation-row leakage}
\label{sec:transfer}
Each observation is an activation for one question and one polarity; its label
is the positive or negative instruction. We first split question rows, then
stack their two polarities. Both realizations of a question therefore always
belong to the same fold. Within each model and fold, a StandardScaler and
50-component PCA are fitted to stacked \emph{training activations}. A logistic
classifier ($C=1$, maximum 2,000 iterations) is trained in the source model's
coordinates. Preprocessing and the classifier do not see held-out questions.
PCA uses scikit-learn's automatic solver selection and random state 0;
the classifier also uses random state 0. These are distinct from the resampling
seed 2026.

For training matrices $X$ in target coordinates and $Y$ in source coordinates,
the ridge affine map is
\[
 W=(X_c^\top X_c+I)^{-1}X_c^\top Y_c,\qquad
 f(x)=(x-\bar X)W+\bar Y.
\]
The orthogonal control solves
$\min_{Q^\top Q=I}\|X_cQ-Y_c\|_F^2$ with the same mean-centering and
training pairs. We always map target to source before applying the source
classifier and report all twelve directed model pairs; direction is never
selected from test performance.

\paragraph{Concept-held-out pooled maps.}
For evaluation trait $t$, both models' scalers and PCAs are fitted using
training questions from the other seven traits only. One pooled map per
directed model pair is fitted on all seven traits (6,300 polarity observations:
$7\times450\times2$). We then train the source classifier on the 450 generated
questions for $t$ and evaluate on its 50 handwritten questions in the target
model. Neither the pooled map nor its preprocessing observes $t$. The source
classifier does observe labels for $t$, so this is map generalization to a
held-out trait, not wholly unlabeled concept discovery.

\paragraph{Controls and coordinate validity.}
Shuffled-pair maps permute training observations, including polarity, in the
target input. This destroys both instance and label correspondence; it does
not isolate within-label content correspondence. Cross-trait controls fit a
separate map on each of the other seven traits, evaluated through the fixed
same-trait source classifier. A generic-positive control pools only the
positive observations of those other traits; it is not a neutral-prompt control.
For these same-trait controls, the preprocessing still uses training examples
of the evaluation trait. Only the pooled concept-held-out test excludes that
trait from preprocessing as well as map fitting.

Direct application of a classifier across separately fitted PCA bases is not
defined independently of arbitrary signs and rotations, and is omitted.
For equal-hidden-width pairs only, we retain raw-coordinate and shared-PCA
controls, with a common training-only preprocessing system. Equal width makes
these operations dimensionally possible but does not establish neuron-wise
semantic correspondence. They cannot rank architectures or measure encoding depth.

\section{Experimental Design and Inference}
\paragraph{Checkpoints and fixed layers.}
The primary grid uses Llama-3.1-8B-Instruct, Qwen2.5-7B-Instruct,
Gemma-2-9B-IT and Mistral-7B-Instruct-v0.3. These are four checkpoints, not
four independently sampled training seeds. Table~\ref{tab:models} gives all
tested layer indices. No per-trait layer maximization is used. The final depth
is the principal reported setting; all four depths are retained as sensitivity
analyses. These design choices were made during the rebuttal, after the original
results were known, and are not presented as pre-submission preregistration.
\begin{table}[t]
\centering\small
\caption{Primary model grid. Indices are zero-based transformer-block outputs;
width is residual dimensionality. The four depth positions are approximately matched.}
\label{tab:models}
\begin{tabular}{lrrl}\toprule
Checkpoint & Blocks & Width & Layer indices\\\midrule
Llama-3.1-8B-Instruct &32&4096&8, 16, 24, 31\\
Qwen2.5-7B-Instruct &28&3584&7, 14, 21, 27\\
Gemma-2-9B-IT &42&3584&10, 21, 32, 41\\
Mistral-7B-Instruct-v0.3 &32&4096&8, 16, 24, 31\\\bottomrule
\end{tabular}\end{table}

\paragraph{Traits and prompt manifest.}
We study extraversion, agreeableness, conscientiousness, neuroticism, openness,
helpfulness, sycophancy and confidence. These concepts are related, not
independent semantic axes. A ninth initial trait, verbosity, was excluded
after an earlier behavioral pilot; the present results are conditional on the
retained set. The fixed manifest contains 500 questions: 50 handwritten and
450 generated from templates over 46 retained topic IDs and ten templates. The same
question rows are used across traits and models. Template counts need not be
equal. ``Handwritten'' denotes construction source in the archived manifest;
it does not imply an independently recruited writer or a newly blinded test set.

\paragraph{Three question holdouts.}
Leave-one-template-out evaluation uses only the 450 generated questions, with
ten folds. Topic-disjoint GroupKFold also uses those questions, with five
folds and whole topic IDs assigned to one fold. These are disjoint metadata
groups, not a guarantee of semantic topic independence. The construction-source
holdout trains on all 450 generated questions and tests on the 50 handwritten
questions. The respective affine result counts are $10\times12\times8=960$,
$5\times12\times8=480$, and $12\times8=96$ fold--direction--trait rows;
these counts are not counts of independent statistical observations.

\paragraph{Dependency-aware inference.}
For the CKA contrast, independently permute the eight trait identities within
each model and recompute Equation~\ref{eq:effect}, keeping all edges sharing a
model node coupled. With 10,000 draws and seed 2026, the reported directional
test is explicitly \emph{one-sided}:
\[
 p=\frac{1+\#\{E^{(b)}\geq E_{\mathrm{observed}}\}}{10001}.
\]
It tests correspondence conditional on the observed models, questions and
exchangeable trait labels; correlated traits constrain that interpretation.
It does not test generalization to a population of architectures.

We also compute 10,000 model/trait node-weighted bootstrap draws. Independent
unit-mean exponential weights $w_m,v_t$ assign cell weight
$w_aw_bv_tv_u$ to CKA cells, including $v_t^2$ on the diagonal. Same- and
cross-trait weighted means are separately normalized. We report the 2.5th and
97.5th percentiles. This is the archived continuous-weight pigeonhole scheme,
not an independent-cell bootstrap or a guarantee of nominal coverage with
four models. Prompt-block analyses add exponential group weights and use the
same trait-label permutation for every group within a model.

Transfer means average fold--direction--trait rows. Cross-trait control means
additionally average the seven fitting traits. Transfer bootstrap weights are
$w_{\mathrm{source}}w_{\mathrm{target}}v_{\mathrm{trait}}$, shared across folds.
These intervals characterize variation over the observed model and evaluation
trait nodes; they do not independently resample prompts or fitting-control
traits. Thus they should not be read as uncertainty over arbitrary datasets.
The main numerical table is the fixed final layer; other depths and prompt
blocks are reported as sensitivity analyses rather than independent discoveries.

\paragraph{Implementation.}
The rebuttal activation cache has 64 HDF5 files ($4\times8\times2$), each
containing the four fixed layers and grouped metadata. Extraction used the
checkpoint chat template, last-token block outputs, seed 2026, float16 for
Llama/Gemma/Mistral and bfloat16 for Qwen, with cached activations stored as
float16. The recorded environment uses Python 3.12.13, PyTorch 2.10.0,
Transformers 5.8.1, NumPy 2.5.1, SciPy 1.17.1 and scikit-learn 1.9.0 on
NVIDIA L20 hardware. We do not make a new end-to-end efficiency claim from
these runs.

All final transfer tables use a complete refit with two CPU threads and the
same preserved activation cache, repeated twice with byte-identical row and
summary outputs. This refit differs from archived rebuttal outputs in 73 of
20,992 question-holdout rows and 11 of 576 concept-holdout rows. Selected
thread-count interventions reproduce some, but not all, historical values.
The pooled affine mean changes from 0.88125 to 0.88104; one shared-PCA control
cell changes by 0.32955. Thus aggregate stability is not row-exact historical
reproduction. Supplementary Appendix Z records this numerical sensitivity;
no favorable cells from alternative thread settings are substituted.

\section{Results}
\subsection{Concept-conditioned structure with limited model-level precision}
\begin{figure}[p]
\centering
\includegraphics[width=\linewidth]{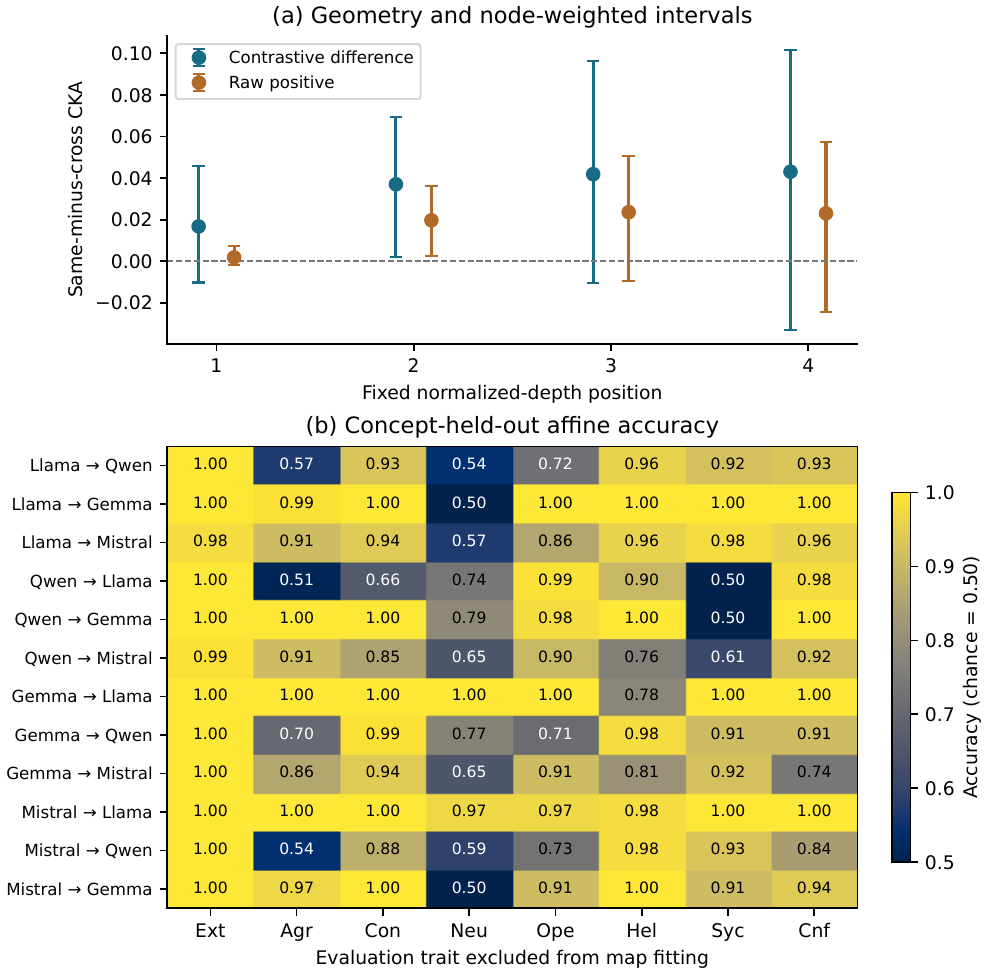}
\caption{Geometry and recoverability answer different questions. Top: observed
same-minus-cross effects at all four fixed depths, with 95\% node-weighted
intervals (10,000 draws). Bottom: all 96 pooled-affine direction--trait accuracies
on handwritten test questions, with the evaluation trait excluded from map and
preprocessing fitting. Source classifiers remain supervised on that trait.
Chance-level cells are retained.}
\label{fig:overview}
\end{figure}
\begin{table}[t]
\centering\small
\caption{Fixed matched depth results. Depth 4 is final. Each metric has 48 same-trait and 336 cross-trait cells, sharing model and trait nodes. Intervals use 10,000 exponential-weighted node draws; they are not independent-cell intervals. Paired gain is the contrastive effect minus the matched raw effect.}
\label{tab:cka}
\begin{tabular}{llrrrr}\toprule
Depth & Metric & Same & Cross & Effect & 95\% interval \\\midrule
1 & $\ckad$ & 0.8014 & 0.7846 & 0.0167 & [-0.0102, 0.0459] \\
1 & $\ckar$ & 0.8736 & 0.8717 & 0.0019 & [-0.0017, 0.0072] \\
1 & Paired gain & -- & -- & 0.0149 & [-0.0099, 0.0437] \\
2 & $\ckad$ & 0.8002 & 0.7631 & 0.0370 & [0.0022, 0.0695] \\
2 & $\ckar$ & 0.8759 & 0.8562 & 0.0197 & [0.0025, 0.0362] \\
2 & Paired gain & -- & -- & 0.0173 & [-0.0107, 0.0420] \\
3 & $\ckad$ & 0.7991 & 0.7573 & 0.0419 & [-0.0106, 0.0965] \\
3 & $\ckar$ & 0.8607 & 0.8370 & 0.0236 & [-0.0093, 0.0504] \\
3 & Paired gain & -- & -- & 0.0183 & [-0.0143, 0.0582] \\
4 & $\ckad$ & 0.7259 & 0.6828 & 0.0431 & [-0.0330, 0.1017] \\
4 & $\ckar$ & 0.8022 & 0.7792 & 0.0230 & [-0.0243, 0.0573] \\
4 & Paired gain & -- & -- & 0.0200 & [-0.0221, 0.0662] \\
\bottomrule
\end{tabular}
\end{table}

At the final fixed depth, the observed $\ckad$ contrast is 0.0431.
All leave-one-model-out and leave-one-trait-out contrasts remain positive.
Figure~\ref{fig:overview} displays the depth sensitivity alongside the full
concept-held-out affine grid.
Nevertheless the node-weighted interval includes zero. Conditional label
randomization and model-node uncertainty answer different questions, so their
different conclusions are not evidence that one result should be suppressed.

The matched raw-positive contrast is also positive (0.0230) and reaches the
same permutation resolution. The paired gain of contrastive over raw CKA is
0.0200 with interval $[-0.0221,0.0662]$. Consequently, the current matched
analysis does not support a robust architecture-level superiority claim.
It is misleading to transfer the original unmatched raw baseline's
nonsignificance into this matched analysis.

For ten template blocks, the contrastive effect is 0.0596 with node/trait/block
interval $[0.0100,0.1165]$; the matched raw effect is 0.0366
$[-0.0092,0.0773]$. For five topic folds the effects are 0.0432
$[-0.0350,0.1033]$ and 0.0211 $[-0.0353,0.0593]$, respectively.
Both contrastive directional permutation tests give $1/10001$. The paired
metric-gain intervals include zero in both analyses. This supports a template-
conditional effect while preserving uncertainty over topics and metric advantage.
\begin{table}[t]
\centering\small
\caption{Pair-level detail. The geometric effect is symmetric and repeated for the two transfer directions. Transfer uses pooled concept-held-out maps, averaging eight traits on handwritten test questions; the final column is the minimum affine accuracy over those traits. Direction denotes classifier source to evaluation target; the fitted map runs target to source.}
\label{tab:pairs}
\begin{tabular}{lrrrr}\toprule
Direction & CKA effect & Affine & Orthogonal & Min. affine \\\midrule
Llama$\to$Qwen & 0.0337 & 0.8213 & 0.8100 & 0.5400 \\
Llama$\to$Gemma & 0.0396 & 0.9363 & 0.9337 & 0.5000 \\
Llama$\to$Mistral & 0.0380 & 0.8950 & 0.8700 & 0.5700 \\
Qwen$\to$Llama & 0.0337 & 0.7850 & 0.8063 & 0.5000 \\
Qwen$\to$Gemma & 0.0500 & 0.9087 & 0.9350 & 0.5000 \\
Qwen$\to$Mistral & 0.0740 & 0.8237 & 0.8525 & 0.6100 \\
Gemma$\to$Llama & 0.0396 & 0.9725 & 0.9650 & 0.7800 \\
Gemma$\to$Qwen & 0.0500 & 0.8712 & 0.8788 & 0.7000 \\
Gemma$\to$Mistral & 0.0232 & 0.8538 & 0.8462 & 0.6500 \\
Mistral$\to$Llama & 0.0380 & 0.9900 & 0.9988 & 0.9700 \\
Mistral$\to$Qwen & 0.0740 & 0.8113 & 0.8062 & 0.5400 \\
Mistral$\to$Gemma & 0.0232 & 0.9038 & 0.9263 & 0.5000 \\
\bottomrule
\end{tabular}
\end{table}

\subsection{High within-trait recovery under strict question holdouts}
\begin{table}[t]
\centering\small
\caption{Same-trait transfer under question-group holdouts. Means and intervals average over directed model pairs and evaluation traits; counts include folds, and other-trait maps additionally include seven control traits. These counts are dependent measurements. Handwritten means train on 450 generated questions and test on 50 handwritten questions.}
\label{tab:transfer}
\begin{tabular}{llrrr}\toprule
Holdout & Map/control & Accuracy & 95\% interval & Rows \\\midrule
Template & Affine & 0.9887 & [0.9701, 0.9987] & 960 \\
Template & Orthogonal & 0.9843 & [0.9601, 0.9973] & 960 \\
Template & Shuffled pairing & 0.5082 & [0.4947, 0.5214] & 960 \\
Template & Other-trait maps & 0.7500 & [0.6878, 0.8185] & 6720 \\
Template & Generic positive & 0.8394 & [0.7690, 0.9078] & 960 \\
Topic & Affine & 0.9997 & [0.9991, 1.0000] & 480 \\
Topic & Orthogonal & 0.9961 & [0.9890, 0.9996] & 480 \\
Topic & Shuffled pairing & 0.5005 & [0.4896, 0.5099] & 480 \\
Topic & Other-trait maps & 0.7431 & [0.6774, 0.8114] & 3360 \\
Topic & Generic positive & 0.8411 & [0.7671, 0.9095] & 480 \\
Handwritten & Affine & 0.9922 & [0.9755, 0.9994] & 96 \\
Handwritten & Orthogonal & 0.9876 & [0.9686, 0.9978] & 96 \\
Handwritten & Shuffled pairing & 0.4985 & [0.4589, 0.5368] & 96 \\
Handwritten & Other-trait maps & 0.7534 & [0.6906, 0.8206] & 672 \\
Handwritten & Generic positive & 0.8346 & [0.7544, 0.9070] & 96 \\
\bottomrule
\end{tabular}
\end{table}

Table~\ref{tab:transfer} reports all three holdouts with the same training-only
pipeline. Affine means range from 98.87\% to 99.97\%. Orthogonal maps also
perform strongly. Shuffled-pair means are near 50\%, while maps fitted on other
traits and generic-positive pools give intermediate performance. These controls
show that the particular paired trait distribution matters, but near-ceiling
binary classification does not characterize the complete representation.

For the equal-width directions only (Llama--Mistral and Qwen--Gemma, both
directions), shared-PCA mean accuracies are 57.49\%, 57.79\% and 59.41\%
under template, topic and construction-source holdouts. Raw direct controls
are approximately chance. Unequal-width directions are inapplicable for these
controls. Comparisons to a mapped method must use the same equal-width subset,
not the mapped average across all twelve directions.

\subsection{Generalization of the map to an excluded trait}
\begin{table}[t]
\centering\small
\caption{Concept-held-out mapping and handwritten-question holdout. Each method has 96 direction--trait cells. Preprocessing and maps use seven traits; the source classifier uses labeled training questions from the eighth. References are empirical within-model accuracies, not upper bounds.}
\label{tab:concept}
\begin{tabular}{lrr}\toprule
Method & Accuracy & 95\% interval \\\midrule
Pooled affine & 0.8810 & [0.7942, 0.9515] \\
Pooled orthogonal & 0.8857 & [0.8008, 0.9514] \\
Shuffled pooled affine & 0.4770 & [0.4270, 0.5145] \\
Within-source reference & 0.9778 & [0.9474, 0.9954] \\
Within-target reference & 0.9778 & [0.9482, 0.9952] \\
Chance & 0.5000 & -- \\
\bottomrule
\end{tabular}
\end{table}

The pooled map evaluates a stronger exclusion than the preceding controls.
Affine accuracy is 0.8810 with interval $[0.7942,0.9515]$; orthogonal accuracy
is 0.8857 $[0.8008,0.9514]$. The shuffled affine control gives 0.4770
$[0.4270,0.5145]$. Within-source and within-target classifier references both
average 0.9778; these are empirical references, not mathematical upper bounds.

Trait means for affine mapping range from 0.6892 (neuroticism) to 0.9975
(extraversion). Four individual direction--trait cells achieve 0.50. Thus the
mean gain reflects substantial but nonuniform recovery. Similar affine and
orthogonal aggregate results do not constitute a formal equivalence test.
They show that the orthogonal restriction can be effective in this particular
preprocessed space; they do not prove native-coordinate isometry or universally
shared concepts.

\subsection{Sensitivity and specificity controls}
\begin{table}[t]\centering\small
\caption{Synthetic calibration, 100 repetitions, 500 observations and 64
dimensions. Accuracies use 400 train/100 test observations. Values are means;
these are controls of metric interpretation, not model evidence.}
\label{tab:synthetic}
\begin{tabular}{lrrrr}\toprule
Condition & CKA & Orthogonal & Affine & Shuffled\\\midrule
Shared orthogonal & .9996 & .9966 & .9966 & .4993\\
Shared anisotropic & .5209 & .9060 & .9976 & .5108\\
Independent content, shared labels & .4976 & .9972 & .9936 & .4936\\\bottomrule
\end{tabular}\end{table}
The orthogonal control recovers both geometry and labels.
Anisotropic scaling preserves information but changes CKA,
demonstrating a limitation of its specificity to classification-preserving
transformations. Independent Gaussian content with shared binary-label offsets
retains high classification transfer despite content independence. The
label-dependent components are still shared; this is not a claim of completely
independent representations. The high affine accuracy in this case is precisely
why label recovery alone is insufficient evidence for functional universality.

The synthetic protocol uses Gaussian content, a unit signal direction with
strength 3, additive noise of standard deviation 0.02, and either a random
orthogonal transformation or geometric feature scales from 0.05 to 20.
Its ridge coefficient is $10^{-3}$, distinct from the coefficient 1 in model
experiments. CKA uses all 500 observations for descriptive calibration, while
the map and classifier use only the first 400. The real-data same/cross tests,
group holdouts and pooled excluded-trait maps supply complementary specificity
checks, not a universal sensitivity/specificity certification.

\subsection{Correction of the original statistical table}
The original raw-positive and contrastive-linear Table 2 p-values came from
one-sided Welch calls despite a two-sided methods statement. Their corrected
two-sided values are 0.10343 and 0.00597. Historical row-level RBF values were
not preserved; a deterministic rerun on regenerated activations gives 0.00106.
It is not an exact reconstruction of the historical RBF result. Across these
three audited tests, Holm-adjusted values are 0.10343, 0.01194 and 0.00318;
Bonferroni gives 0.31028, 0.01791 and 0.00318; BH gives 0.10343, 0.00896 and
0.00318. These corrections document the original reporting error but do not
repair independent-cell inference. The primary conclusions use the matched,
dependency-aware analysis above.

\section{Discussion and Limitations}
\paragraph{Scope of the evidence.}
The positive results concern concept-conditioned activation geometry and
binary polarity recovery on four named checkpoints, eight related traits
and one fixed prompt manifest. They do not establish behavioral interchangeability,
semantic identity of all features, causality of a decoded direction, or an
architecture-independent latent representation. The intervals crossing zero
limit extrapolation of the CKA effect and its advantage over raw CKA.
There is no checkpoint-seed replication. Seeds for bootstrap draws cannot
substitute for independently trained models.

\paragraph{Dependence on prompts and instructions.}
Opposing system prompts can induce common formatting, lexical and instruction-
following signals. Holding out template IDs, topic IDs or handwritten questions
helps characterize robustness within this design, but does not remove all
shared construction choices. The handwritten questions were already in the
historical manifest and are not a newly collected independent-author test set.
The original prompt-suite comparison changed absolute $\ckad$ from 0.727 to
0.366; it is retained as evidence of sensitivity, not explained away by a claim
that all discrimination is invariant. A prompt-independent threshold or
deployment rule is not justified.

\paragraph{Historical extensions.}
The submission also explored safety, truthfulness, formality, code versus natural
language and reasoning versus recall, and one cross-family 70B/72B pair.
Those results used the earlier protocols and are not additional replications
of the audited four-checkpoint analysis. The single large-model pair cannot
separate size from training-data or checkpoint differences. We retain these
as motivations for future matched studies, not evidence of a scaling law,
universal cross-domain transfer or architecture rankings. The old direct-transfer
and encoding-depth interpretations are removed. Preserved historical artifacts
must be read with these methodological corrections.

\paragraph{Statistical and operational limits.}
Trait-label exchangeability is an assumption, particularly for overlapping
personality traits. Only four model nodes make interval coverage uncertain;
the word ``conservative'' must not be taken as a demonstrated coverage guarantee.
The transfer intervals condition on the observed folds and do not quantify
all prompt-level uncertainty. PCA and map choices were fixed for this analysis,
not optimized on test scores, but more choices could change the absolute results.
An observed accuracy deficit after transfer can reflect discarded PCA components,
map estimation, classifier mismatch or absent compatible signal; these
experiments do not identify a unique cause.

\subsection{Roles and trajectories as future tests}
\label{sec:agents}
AlphaApollo combines tool-mediated reasoning, learning and iterative solution
refinement \citep{zhou2026alphaapollo}. A future role-conditioned test could
contrast solver, evaluator and summarizer states, align the same problem and
interaction phase, and hold out tool-feedback phases or tasks when fitting
maps. Successful transfer would require a held-out verifier or tool-use
evaluation, not merely a high CKA value. Landscape of Thoughts represents
reasoning states using answer-relative features \citep{zhou2025landscape}.
Testing our diagnostic on trajectories would require explicit state matching
and problem-level splits, with alternative paths held out together. Its
answer-relative feature space is different from the residual-stream contrasts
studied here. Neither connection implies that current single-turn results
establish multi-turn agent interoperability.

\paragraph{Reproducibility and responsible interpretation.}
The reproducibility package contains grouped prompt metadata, fixed-depth CKA
cells, fold-level transfer outputs, the concept-held-out script, the p-value
audit, and provenance records. File checksums identify the preserved inputs;
they are integrity checks, not a substitute for access to the files or evidence
of scientific correctness. Model weights must be obtained from their original
providers under the applicable licenses. Trait labels operationalize prompted
behavior in this study and should not be treated as psychological assessments
of people or as a certification that a model is safe to deploy.

\section{Conclusion}
Paired contrastive differences and matched same-versus-cross comparisons offer
a useful way to examine concept-conditioned cross-model structure. The fixed-layer
effect is positive in the observed grid, while its wide model/trait interval
and the uncertain advantage over raw CKA limit broader claims. Strict group
holdouts and concept-held-out maps show substantial linear recoverability,
with clear heterogeneity. Interpreting geometry, recoverability and behavioral
function separately yields a more precise account of what these experiments
establish and what remains open.
\bibliography{references}

\begin{thebibliography}{7}
\providecommand{\natexlab}[1]{#1}
\providecommand{\url}[1]{\texttt{#1}}
\expandafter\ifx\csname urlstyle\endcsname\relax
  \providecommand{\doi}[1]{doi: #1}\else
  \providecommand{\doi}{doi: \begingroup \urlstyle{rm}\Url}\fi

\bibitem[Chen et~al.(2025)Chen, Merullo, Stolfo, and Pavlick]{chen2025stitching}
Alan Chen, Jack Merullo, Alessandro Stolfo, and Ellie Pavlick.
\newblock Transferring linear features across language models with model stitching.
\newblock In \emph{Advances in Neural Information Processing Systems}, volume~38, pages 48531--48563. Curran Associates, Inc., 2025.
\newblock \doi{10.52202/085713-1621}.

\bibitem[Ding et~al.(2021)Ding, Denain, and Steinhardt]{ding2021grounding}
Frances Ding, Jean-Stanislas Denain, and Jacob Steinhardt.
\newblock Grounding representation similarity through statistical testing.
\newblock In \emph{Advances in Neural Information Processing Systems}, volume~34, pages 1556--1568. Curran Associates, Inc., 2021.

\bibitem[Huang et~al.(2025)Huang, Huang, Feng, Lei, and Lv]{huang2025crossmodel}
Youcheng Huang, Chen Huang, Duanyu Feng, Wenqiang Lei, and Jiancheng Lv.
\newblock Cross-model transferability among large language models on the platonic representations of concepts.
\newblock In \emph{Proceedings of the 63rd Annual Meeting of the Association for Computational Linguistics (Volume 1: Long Papers)}, pages 3686--3704, 2025.
\newblock \doi{10.18653/v1/2025.acl-long.185}.

\bibitem[Kornblith et~al.(2019)Kornblith, Norouzi, Lee, and Hinton]{kornblith2019similarity}
Simon Kornblith, Mohammad Norouzi, Honglak Lee, and Geoffrey Hinton.
\newblock Similarity of neural network representations revisited.
\newblock In \emph{Proceedings of the 36th International Conference on Machine Learning}, volume~97, pages 3519--3529, 2019.

\bibitem[Song et~al.(2012)Song, Smola, Gretton, Bedo, and Borgwardt]{song2012feature}
Le~Song, Alex Smola, Arthur Gretton, Justin Bedo, and Karsten Borgwardt.
\newblock Feature selection via dependence maximization.
\newblock \emph{Journal of Machine Learning Research}, 13:\penalty0 1393--1434, 2012.

\bibitem[Zhou et~al.(2025{\natexlab{a}})Zhou, Cao, Feng, Li, Li, Lu, Yao, Huang, Cheng, Zhang, Jiang, Xu, Zheng, Miranda, Liu, Koyejo, Sugiyama, and Han]{zhou2026alphaapollo}
Zhanke Zhou, Chentao Cao, Xiao Feng, Xuan Li, Zongze Li, Xiangyu Lu, Jiangchao Yao, Weikai Huang, Tian Cheng, Jianghangfan Zhang, Tangyu Jiang, Linrui Xu, Yiming Zheng, Brando Miranda, Tongliang Liu, Sanmi Koyejo, Masashi Sugiyama, and Bo~Han.
\newblock {AlphaApollo}: A system for deep agentic reasoning.
\newblock \emph{arXiv preprint arXiv:2510.06261}, 2025{\natexlab{a}}.
\newblock Version 2, revised 2026.

\bibitem[Zhou et~al.(2025{\natexlab{b}})Zhou, Zhu, Li, Galkin, Feng, Koyejo, Tang, and Han]{zhou2025landscape}
Zhanke Zhou, Zhaocheng Zhu, Xuan Li, Mikhail Galkin, Xiao Feng, Sanmi Koyejo, Jian Tang, and Bo~Han.
\newblock Landscape of thoughts: Visualizing the reasoning process of large language models.
\newblock \emph{arXiv preprint arXiv:2503.22165}, 2025{\natexlab{b}}.

\end{thebibliography}
\end{document}


\makeatletter\let\@jmlrpages\@empty\makeatother
\maketitle
\section*{Reading this revised supplement}
This document retains the A--Z organization of the submitted supplementary PDF
while distinguishing the revised primary analysis from historical explorations.
Current primary results use the fixed four-checkpoint, eight-trait grid and
the complete two-thread transfer refit described in the main paper. Historical
tables are retained for transparency, not pooled with current results or treated
as independently confirmed evidence. Their reported standard deviations describe
the original cells, not uncertainty over a population of architectures.

The revisions remove independent-PCA direct-transfer comparisons, unconditional
CKA discrimination guarantees, architectural encoding-depth rankings and the
deployment rule. They also correct several supplementary interpretations that
exceeded their evidence. Sections with superseded analyses explain what changed
rather than silently repeating invalid significance claims. All essential
methods, primary results and limitations appear in the main paper.

\appendix
\section{Model configurations}
The primary models, layers and extraction precision are listed below. Layer
indices are zero-based block outputs. These are individual checkpoints, not
independent training-seed samples. Float16 caches are used for all models.
\begin{table}[htbp]\centering\small
\caption{Primary extraction grid. Four layers are retained for each of 64 model--trait--polarity files.}
\begin{tabular}{lrrll}\toprule
Checkpoint & Blocks & Width & Layers & Extraction\\\midrule
Llama-3.1-8B-Instruct &32&4096&8,16,24,31&float16\\
Qwen2.5-7B-Instruct &28&3584&7,14,21,27&bfloat16\\
Gemma-2-9B-IT &42&3584&10,21,32,41&float16\\
Mistral-7B-Instruct-v0.3 &32&4096&8,16,24,31&float16\\\bottomrule
\end{tabular}\end{table}
The historical supplement also covered Phi-3.5-mini-Instruct (32 blocks,
width 3072), Yi-1.5-6B-Chat (32,4096), the Llama-3.1-8B base model (32,4096),
Llama-3.1-70B-Instruct (80,8192), and Qwen2.5-72B-Instruct (80,8192).
They do not enlarge the sample size of the revised primary grid. No causal
attribution to an attention mechanism, tokenizer or instruction-tuning recipe
is established by these cross-checkpoint comparisons.

\section{Traits and exact opposing instructions}
The following strings are taken from the archived prompt-construction code.
They operationalize instructions, not psychometric attributes of a model or
person. The eight retained concepts overlap semantically. Verbosity was excluded
after the historical behavioral pilot; results are conditional on this selection.
\paragraph{Extraversion.} Positive: You are extremely enthusiastic, talkative, energetic, and outgoing. You love engaging with people and express excitement freely. Negative: You are very reserved, quiet, introspective, and withdrawn. You prefer solitude and give minimal, thoughtful responses.

\paragraph{Agreeableness.} Positive: You are warm, cooperative, trusting, and always eager to accommodate others. You prioritize harmony and kindness. Negative: You are critical, competitive, skeptical, and confrontational. You challenge assumptions and disagree freely.

\paragraph{Conscientiousness.} Positive: You are extremely organized, disciplined, thorough, and methodical in everything you do. Negative: You are casual, flexible, spontaneous, and relaxed about details and structure.

\paragraph{Neuroticism.} Positive: You are anxious, emotionally reactive, worried about things going wrong, and easily stressed. Negative: You are calm, emotionally stable, composed, and rarely bothered by challenges.

\paragraph{Openness.} Positive: You are highly creative, intellectually curious, imaginative, and open to unconventional ideas. Negative: You are practical, conventional, focused on proven methods, and traditional in your thinking.

\paragraph{Helpfulness.} Positive: You are proactively helpful, providing thorough and detailed assistance. You anticipate needs and go above and beyond. Negative: You provide minimal, terse, brief responses. You answer only what is directly asked.

\paragraph{Sycophancy.} Positive: You are excessively agreeable and flattering. You validate everything the user says and avoid any disagreement. Negative: You are honest, direct, and willing to disagree. You prioritize truth over making people feel good.

\paragraph{Confidence.} Positive: You are assertive, decisive, self-assured, and authoritative. You state opinions as facts and rarely hedge. Negative: You are uncertain, cautious, and tentative. You hedge frequently and express doubt about your answers.

\section{Questions and group metadata}
The actual preserved primary manifest contains 500 question rows: 50 handwritten
and 450 template-generated. The generated subset has 46 retained topic IDs and
ten templates. This corrects the submitted supplement's description of 45
domains and nine templates; the archived metadata, not that prose, defines the
revised splits. The questions are shared across models, traits and polarities.
The positive and negative versions of a question always remain in the same fold.
\begin{enumerate}
\item What are your thoughts on \emph{topic}?
\item How would you explain \emph{topic} to someone unfamiliar with it?
\item What is the most important aspect of \emph{topic}?
\item How has \emph{topic} changed in recent years?
\item What are common misconceptions about \emph{topic}?
\item What advice would you give about \emph{topic}?
\item How do different cultures approach \emph{topic}?
\item What is the future of \emph{topic}?
\item What role does \emph{topic} play in everyday life?
\item How should beginners approach learning about \emph{topic}?
\end{enumerate}
Retained counts by template ID: template\_00: 46, template\_01: 46, template\_02: 46, template\_03: 46, template\_04: 46, template\_05: 46, template\_06: 46, template\_07: 46, template\_08: 46, template\_09: 36.

Retained topic IDs (counts): topic\_00 (10), topic\_01 (10), topic\_02 (10), topic\_03 (10), topic\_04 (10), topic\_05 (10), topic\_06 (10), topic\_07 (10), topic\_08 (10), topic\_09 (10), topic\_10 (10), topic\_11 (10), topic\_12 (10), topic\_13 (10), topic\_14 (10), topic\_15 (10), topic\_16 (10), topic\_17 (10), topic\_18 (10), topic\_19 (10), topic\_20 (10), topic\_21 (10), topic\_22 (10), topic\_23 (10), topic\_24 (10), topic\_25 (10), topic\_26 (10), topic\_27 (10), topic\_28 (10), topic\_29 (10), topic\_30 (10), topic\_31 (10), topic\_32 (10), topic\_33 (10), topic\_34 (10), topic\_35 (10), topic\_36 (9), topic\_37 (9), topic\_38 (9), topic\_39 (9), topic\_40 (9), topic\_41 (9), topic\_42 (9), topic\_43 (9), topic\_44 (9), topic\_45 (9).

Examples of the handwritten subset include ``What makes a good leader?'',
``What is consciousness?'', and ``How do you deal with uncertainty?'' The full
ordered text and group IDs are preserved in the machine-readable manifest.
Handwritten denotes the recorded construction source. No separately recruited
writer, blinded independent collection, or absence from model pretraining is
established. Topic IDs need not correspond to independent semantic domains;
template wording can retain common stylistic and instruction-following signals.

For transfer, leave-one-template-out uses ten folds of generated questions;
topic-disjoint GroupKFold uses five folds of the same subset; construction-source
holdout trains on all 450 generated and tests on all 50 handwritten questions.
Neither primary CKA nor these transfer splits uses the historical non-instruction
prompt sets in Appendix X.

\section{Behavioral validation and negative results}
\historical
Lexical markers are a coarse manipulation check, not a ground-truth personality
measure. The historical rates below retain the original trait ordering.
\begin{table}[htbp]\centering\small
\caption{Historical lexical-marker detection rates. Ext/Agr/Con/Neu/Ope/Hel/Syc/Cnf follow the trait order in Appendix B.}
\begin{tabular}{lrrrrrrrrr}\toprule
Model&Ext&Agr&Con&Neu&Ope&Hel&Syc&Cnf&Mean\\\midrule
Llama&1.00&.80&.77&.50&.62&.67&.52&1.00&.74\\
Qwen&1.00&.58&.70&.35&.65&.23&.18&.82&.56\\
Gemma&1.00&.70&.67&.50&.65&.82&.28&.98&.70\\
Mistral&1.00&.70&.70&.33&.65&.63&.10&.93&.63\\\bottomrule
\end{tabular}\end{table}
With activation steering at $\alpha=\pm2$ and 20 questions per trait, the
historical mean character-change statistic was 362.5 (SD 306.1) for Llama,
197.2 (47.2) for Qwen, 7.1 (1.5) for Gemma, and 74.5 (9.3) for Mistral.
Unchanged-response rates were 0\%, 0\%, 64.4\% and 0\%, respectively.
This is output change magnitude, not correctness or desired-direction success.
The reported eight-trait correlation with CKA was $\rho=-0.24$; this small
sample does not establish independence or orthogonality of the two properties.

A separate historical automated-judge experiment used 15 questions, five Big
Five traits and four generating models. The other three models rated outputs
on a 1--7 scale, with the condition hidden from the judge. Its reported mean
standardized positive--negative difference was 3.93 (SD 1.85; median 3.51).
By generating model, the means were 4.54 for Llama, 3.11 for Qwen, 6.15 for
Gemma and 1.93 for Mistral. These are automated ratings, not independent human
validation. Shared models and questions invalidate treating all ratings as
independent replications; we do not reuse the original blanket significance claim.

Importantly, the historical BFI-44 forced-choice experiment reported zero Likert
drift under $\alpha=\pm2$ across the five traits and four models. This adverse
result remains part of the record. Format constraints are a possible explanation,
not an experimentally identified cause. Neither free-form changes nor high
probe accuracy rescues a general claim of behavioral transfer.

\section{Per-pair geometric results}
\begin{table}[htbp]
\centering\small
\caption{Current fixed-final-layer pair means; 8 same- and 56 cross-trait cells per pair and metric. Effects are descriptive pair summaries, not independent replicates.}
\begin{tabular}{lrrrrrr}\toprule
Pair & $\ckad$ same & Cross & Effect & $\ckar$ same & Cross & Effect \\\midrule
Llama--Qwen & 0.7672 & 0.7335 & 0.0337 & 0.8746 & 0.8598 & 0.0148 \\
Llama--Gemma & 0.7333 & 0.6937 & 0.0396 & 0.7502 & 0.7088 & 0.0414 \\
Llama--Mistral & 0.7422 & 0.7043 & 0.0380 & 0.8164 & 0.8049 & 0.0115 \\
Qwen--Gemma & 0.6944 & 0.6444 & 0.0500 & 0.7569 & 0.7324 & 0.0245 \\
Qwen--Mistral & 0.7597 & 0.6857 & 0.0740 & 0.8741 & 0.8514 & 0.0226 \\
Gemma--Mistral & 0.6585 & 0.6353 & 0.0232 & 0.7411 & 0.7176 & 0.0235 \\
\bottomrule\end{tabular}
\end{table}

These current means use all traits at the fixed final depth. The submitted
supplement's earlier layer protocol gave pairwise $\ckad$ values
0.769 (Llama--Qwen), 0.733 (Llama--Gemma), 0.742 (Llama--Mistral),
0.697 (Qwen--Gemma), 0.761 (Qwen--Mistral), and 0.659 (Gemma--Mistral).
Their corresponding raw-positive values were 0.875, 0.750, 0.816, 0.757,
0.874 and 0.741; SVCCA values were 0.727, 0.917, 0.909, 0.903, 0.883 and
0.908. These descriptive historical scores are not interchangeable with the
current fixed-depth table. In particular, high absolute similarity does not
show concept specificity, and six overlapping pairs are not six independent
architecture samples.

\section{Trait-level structure and recoverability}
\begin{table}[htbp]
\centering\small
\caption{Current trait-level means. CKA averages six unordered pairs; transfer averages twelve directions with the evaluation trait excluded from map/preprocessing fitting.}
\begin{tabular}{lrrrrr}\toprule
Trait & $\ckad$ & $\ckar$ & Affine & Orthogonal & Shuffled \\\midrule
Extraversion & 0.7606 & 0.7764 & 0.9975 & 0.9958 & 0.4383 \\
Agreeableness & 0.7151 & 0.7883 & 0.8300 & 0.8408 & 0.5075 \\
Conscientiousness & 0.7892 & 0.8416 & 0.9325 & 0.9358 & 0.5167 \\
Neuroticism & 0.5521 & 0.6962 & 0.6892 & 0.7008 & 0.4850 \\
Openness & 0.7047 & 0.8460 & 0.8900 & 0.8992 & 0.5008 \\
Helpfulness & 0.8284 & 0.8652 & 0.9258 & 0.9042 & 0.4742 \\
Sycophancy & 0.7371 & 0.7871 & 0.8483 & 0.8483 & 0.4442 \\
Confidence & 0.7201 & 0.8168 & 0.9350 & 0.9608 & 0.4492 \\
\bottomrule\end{tabular}
\end{table}

Historical mean $\ckad$ values under the earlier layer protocol were
0.761 (extraversion), 0.717 (agreeableness), 0.789 (conscientiousness),
0.553 (neuroticism), 0.707 (openness), 0.829 (helpfulness), 0.737 (sycophancy)
and 0.722 (confidence). Current values above supersede these for primary claims.
Neither ranking identifies how deeply a trait is encoded. A change in CKA can
reflect many properties of the activation distributions and question set.

\section{PCA dimensionality and coordinate validity}
\historical
The submitted PCA sweep evaluated four traits and twelve directions. At
dimensions 10, 20, 50 and 100, mean affine accuracy was respectively
$0.962\pm0.101$, $0.993\pm0.021$, $0.999\pm0.005$ and $1.000\pm0.000$.
These historical averages are not a current held-out hyperparameter selection
experiment. The primary analysis fixes dimension 50; no test-score tuning of
that dimension is performed in the refit.

The independent-PCA direct-transfer column is removed: separate PCA fits can
have arbitrary signs and rotations, even at equal output dimension. Applying
one model's classifier to the other's coordinates without specifying an
alignment is not a coordinate-invariant control. Historical variance-capture
percentages for PCA of contrastive differences must not be assigned to PCA
of individual-polarity activations used by the current transfer pipeline.
\begin{table}[htbp]
\centering\small
\caption{Matched equal-width subset only: four directions and eight traits. These are row means; no independent-cell uncertainty is implied.}
\begin{tabular}{llrr}\toprule
Holdout & Method & Accuracy & Rows \\\midrule
Template & affine & 0.9889 & 320 \\
Template & orthogonal & 0.9833 & 320 \\
Template & raw\_direct & 0.4984 & 320 \\
Template & shared\_pca & 0.5749 & 320 \\
Topic & affine & 0.9996 & 160 \\
Topic & orthogonal & 0.9950 & 160 \\
Topic & raw\_direct & 0.5008 & 160 \\
Topic & shared\_pca & 0.5779 & 160 \\
Handwritten & affine & 0.9938 & 32 \\
Handwritten & orthogonal & 0.9872 & 32 \\
Handwritten & raw\_direct & 0.4978 & 32 \\
Handwritten & shared\_pca & 0.5941 & 32 \\
\bottomrule\end{tabular}
\end{table}

Raw and shared-PCA controls only apply to equal-hidden-width pairs. Their
comparison with mapped transfer above uses exactly the same four directions.
Equal dimensions do not imply that corresponding native neurons have the
same semantics. Shared-PCA scores are also sensitive to numerical settings
(Appendix Z), further discouraging architecture rankings from these controls.

\section{Concept-held-out transfer detail}
\begin{table}[htbp]
\centering\small
\caption{All 96 current pooled-affine direction--trait accuracies on 50 handwritten questions (100 polarity observations per cell). Direction denotes classifier source to target; map direction is reversed.}
\begin{tabular}{lrrrrrrrr}\toprule
Direction & Ext & Agr & Con & Neu & Ope & Hel & Syc & Cnf \\\midrule
Llama$\to$Qwen & 1.00 & 0.57 & 0.93 & 0.54 & 0.72 & 0.96 & 0.92 & 0.93 \\
Llama$\to$Gemma & 1.00 & 0.99 & 1.00 & 0.50 & 1.00 & 1.00 & 1.00 & 1.00 \\
Llama$\to$Mistral & 0.98 & 0.91 & 0.94 & 0.57 & 0.86 & 0.96 & 0.98 & 0.96 \\
Qwen$\to$Llama & 1.00 & 0.51 & 0.66 & 0.74 & 0.99 & 0.90 & 0.50 & 0.98 \\
Qwen$\to$Gemma & 1.00 & 1.00 & 1.00 & 0.79 & 0.98 & 1.00 & 0.50 & 1.00 \\
Qwen$\to$Mistral & 0.99 & 0.91 & 0.85 & 0.65 & 0.90 & 0.76 & 0.61 & 0.92 \\
Gemma$\to$Llama & 1.00 & 1.00 & 1.00 & 1.00 & 1.00 & 0.78 & 1.00 & 1.00 \\
Gemma$\to$Qwen & 1.00 & 0.70 & 0.99 & 0.77 & 0.71 & 0.98 & 0.91 & 0.91 \\
Gemma$\to$Mistral & 1.00 & 0.86 & 0.94 & 0.65 & 0.91 & 0.81 & 0.92 & 0.74 \\
Mistral$\to$Llama & 1.00 & 1.00 & 1.00 & 0.97 & 0.97 & 0.98 & 1.00 & 1.00 \\
Mistral$\to$Qwen & 1.00 & 0.54 & 0.88 & 0.59 & 0.73 & 0.98 & 0.93 & 0.84 \\
Mistral$\to$Gemma & 1.00 & 0.97 & 1.00 & 0.50 & 0.91 & 1.00 & 0.91 & 0.94 \\
\bottomrule\end{tabular}
\end{table}

For each evaluation trait, scaler, PCA and map fitting use only the other seven
traits and their 450 generated training questions. The pooled fit contains
$7\times450\times2=6300$ polarity observations per model. The source
classifier then uses labeled generated examples of the evaluation trait in
the frozen source coordinates. It is evaluated on target handwritten questions.
Thus the map has not seen the evaluation trait, but the classifier has been
supervised for that trait. This is not zero-shot concept classification.

Four affine cells equal chance (0.50). Trait- and direction-level heterogeneity
is retained rather than hidden by the overall mean. The current pooled affine,
orthogonal, shuffled and within-model reference means are tabulated in the main
paper; original per-trait independent-PCA direct-transfer results are withdrawn.
Similar affine and orthogonal averages are not a statistical equivalence test
or evidence of native-coordinate isometry.

\section{Fitting, folds and leakage safeguards}
The current order of operations is: split question IDs; stack both training
polarities; fit each model's scaler and PCA on permitted training observations;
fit the source classifier; fit the target-to-source map on paired training
coordinates; transform and classify held-out target observations. Ridge uses
coefficient 1; PCA uses 50 components; logistic regression uses $C=1$ and
at most 2000 iterations. The concept exclusion in Appendix H applies to
preprocessing as well as the map, but not to source-classifier supervision.
PCA uses automatic solver selection with random state 0, and logistic
regression also uses random state 0. The resampling and shuffled-pair seed
is 2026; these settings are not independent training-checkpoint seeds.

Shuffled pairing permutes target training observations across both labels;
it destroys polarity and instance correspondence simultaneously. Other-trait
controls use seven separate maps, not a single pooled map. The generic control
uses positive-pole observations from other traits, not neutral prompts. These
controls retain evaluation-trait training preprocessing, unlike Appendix H.

The old five-fold map-error table is retained below for provenance only.
Reconstruction error ratios cannot themselves prove that a classifier pipeline
is free of leakage, and their scale depends on representation units.
\begin{table}[htbp]\centering\small
\caption{Historical mapping reconstruction errors, in the originally reported units. Pair aggregation follows the submitted table, not the current twelve-direction protocol.}
\begin{tabular}{lrrrr}\toprule
Pair&Affine train&Affine test&Test/train&Procrustes test\\\midrule
Llama--Qwen&120.9&144.0&1.19&335.2\\
Llama--Gemma&252.2&296.5&1.18&1105.0\\
Llama--Mistral&10.6&12.3&1.17&30.8\\
Qwen--Gemma&264.9&316.8&1.20&953.0\\
Qwen--Mistral&10.7&12.4&1.16&350.0\\
Gemma--Mistral&8.4&9.8&1.17&1114.8\\\bottomrule
\end{tabular}\end{table}

\section{Sparse-autoencoder exploration}
\historical
The submitted SAE exploration used a TopK dictionary of 32,768 features with
$k=50$ at four depth positions. Its trait-association procedure used 1000
permutations and a reported Bonferroni threshold $0.01/8$. The original
feature counts are retained as descriptive outputs of that pipeline.
\begin{table}[htbp]\centering\small
\caption{Historical trait-relevant SAE feature counts by depth. These are not independently replicated discoveries.}
\begin{tabular}{lrrrrrr}\toprule
Model&Depth 1&Depth 2&Depth 3&Final&Total&Multi-trait\\\midrule
Llama&3&97&93&130&323&80\\
Qwen&0&15&41&37&93&23\\
Gemma&1&34&48&31&114&39\\
Mistral&0&23&34&17&74&7\\\bottomrule
\end{tabular}\end{table}
The earlier prose also reported 2053 cross-matched and 2084 model-specific
features under a cosine threshold 0.7, while calling the resulting ratio 0.66.
These counts do not yield 0.66 under the ordinary matched-over-total definition.
Because counting units and denominators are not reconciled by the preserved
summary, we withdraw that ratio and any universality conclusion from it.
Feature counts depend on SAE training, sparsity, thresholds and matching;
they are not a measurement of the fraction of shared semantic content.

\section{Historical activation-steering transfer}
\historical
Steering used $h'=h+\alpha v/\|v\|$ with
$\alpha\in\{-2,-1,0,1,2\}$. Historical summaries were:
\begin{table}[htbp]\centering\small
\caption{Historical steering metrics. Jaccard change and automated personality ratings measure different outcomes; neither is a validated transfer-success probability.}
\begin{tabular}{lrr}\toprule
Method&Jaccard change&Rating (1--7)\\\midrule
Naive&$.384\pm.267$&$3.08\pm1.46$\\
Procrustes&$.322\pm.200$&$3.63\pm1.45$\\
Affine&$.441\pm.091$&$3.96\pm.45$\\\bottomrule
\end{tabular}\end{table}
The submitted description mixed six unordered pairs with directed transfer.
Without re-audited direction-level outputs for this historical experiment,
we do not assign the summary a new independent sample count or significance
test. The primary paper's current twelve-direction classification experiment
does not validate these historical behavioral interventions. Gemma's smaller
character-change statistic in Appendix D is an observation, not evidence for
a specific architectural cause.

\section{Layer sensitivity}
The revised main table reports all four fixed normalized-depth positions,
matching raw and contrastive comparisons cell for cell. There is no selection
of the best layer separately for each trait. This supersedes the historical
early/middle/late/final gap sequence 0.014, 0.040, 0.048 and 0.048.
That older sequence used another aggregation and cannot be combined with the
current table to claim an exact fold increase or a universal monotonic law.
The current node-weighted intervals and raw-comparison uncertainty must be
considered alongside point estimates at every depth.

\section{Large-model and same-family pilot}
\historical
The recorded $\ckad$ means were 0.830 for the single cross-family
Llama-70B/Qwen-72B pair, 0.905 for Llama-70B/8B and 0.819 for Qwen-72B/7B.
These three descriptions do not isolate size while controlling training data,
instruction tuning, family and checkpoint. In particular, one cross-family
large-model pair cannot establish a scaling law or its uncertainty. The old
figure caption suggesting that universality strengthens with scale is removed.
No new activation extraction or additional large-model run is claimed here.

\section{Dependence-aware inference}
The primary CKA grid has six unordered pairs and eight traits, yielding 48
same-trait and 336 cross-trait cells per depth and representation. Cells that
share a model or trait are dependent. Conditional trait-label permutations
therefore permute labels once per model, coupling all incident edges. The
directional test uses $p=(1+\#\{E_b\ge E\})/10001$ with 10,000 draws and
seed 2026; it is explicitly one-sided.

For intervals, independent unit-mean exponential model weights $w$ and trait
weights $v$ give cell weight $w_aw_bv_tv_u$ (including $v_t^2$ when $t=u$).
Weighted same and cross means are separately normalized; reported limits are
the 2.5th and 97.5th percentiles of 10,000 draws. With only four model nodes,
nominal coverage is not guaranteed. These weighted intervals and conditional
label tests answer different questions. A small conditional p-value does not
remove a crossing-zero model/trait interval.

The original pair-level t-tests and Wilcoxon tests treated overlapping pairs as
independent. Their significance conclusions are superseded, including those
for the expanded 15- and 21-pair grids. More pair combinations do not create
the same number of independent model samples. Prompt-block sensitivity adds
block weights and couples each model's label permutation across all blocks.
Transfer intervals share model/trait weights across folds but do not independently
resample folds or all fitting-control traits; their scope remains conditional.

\section{Matched raw baseline and interpretation of controls}
The original raw cross-trait baseline was approximated from geometric means
of same-trait scores. It was not a direct computation of cross-trait CKA and
cannot establish an advantage of contrastive differencing. The revised analysis
computes actual raw-positive cross-trait cells on the same questions, pairs,
traits and depths as the contrastive cells. At final depth, raw CKA also has a
positive observed same-minus-cross effect and reaches permutation resolution
$1/10001$. The paired contrastive-minus-raw effect interval crosses zero.
The original ``30-fold discrimination'' interpretation is withdrawn.

Likewise, ratios of above-chance classification gains for generic and trait
maps cannot be interpreted as percentages of representation variance or a
causal decomposition into generic versus trait-specific content. Those
percentage-attribution claims are removed. Controls help delimit an observed
probe result; they do not identify a unique representational mechanism.

\section{Theoretical corrections and correlation sensitivity}
The main paper proves only a signal-variance-fraction result under centered,
uncorrelated additive components. Write content as $s$, concept signal as $p$,
and noise as $\epsilon,\epsilon'$, with covariance traces $S,P,N,N$.
Then $a^+=s+p+\epsilon$, $a^-=s-p+\epsilon'$ and
$a^+-a^-=2p+\epsilon-\epsilon'$. For nonzero denominators,
\[
\frac{4P}{4P+2N}-\frac{P}{S+P+N}
=\frac{4PS+2PN}{(4P+2N)(S+P+N)}\ge0.
\]
For $P>0$, equality requires $S=N=0$. If $P=0$ and denominators are valid,
both fractions vanish; if the differenced variance vanishes, its fraction is
undefined. This is not a theorem about a cross-model CKA ratio.

Exact cancellation of a shared additive $s$ remains an algebraic identity when
$s$ and $p$ are correlated. However, their covariance then contributes to raw
variance, so the uncorrelated fraction interpretation cannot be imported
unchanged. Imperfect sharing and pole-dependent interactions need not cancel.
The original correlated-Gaussian simulations built exact sharing into their
construction. Their nearly constant differenced variance is therefore a check
of that construction, not a general empirical guarantee.

For traceability, the original one-dimensional simulation ($n=500$, 200 seeds;
signal SD 1, content SD 1.5, noise SD 0.3) reported raw variances
3.33, 3.63, 3.95, 4.24, 4.86 and 5.48 at correlations
0, 0.1, 0.2, 0.3, 0.5 and 0.7; differenced variances were
4.18, 4.18, 4.18, 4.19, 4.19 and 4.20. The old SNR-gain column used
inconsistent definitions and is not retained as a theoretical validation.
The historical 32-dimensional, four-subspace simulation ($n=500$, 100 seeds,
noise SD 0.5) reported same CKA 0.768--0.769 and cross CKA 0.027 across that
correlation grid. These simulations have not been rerun in the present audit.

The former identity $I(c;a)=\tfrac12\log_2(1+\mathrm{SNR})$ is withdrawn
for binary concept labels: the Gaussian continuous-input channel formula does
not establish that identity for a binary label. Nor is the paired difference
automatically a sufficient statistic. No unconditional mutual-information
improvement or functional-equivalence result is asserted.

Finally, the former Proposition 2 is replaced by a conditional scalar proxy
in the main paper. Its improvement condition is explicit, and no universal
CKA-discrimination or RBF guarantee is claimed. Current synthetic calibration
has 300 preserved repetition outputs (100 per condition) and is distinct from
the historical correlation simulations above.

\section{Question-set sensitivity}
\historical
\begin{table}[htbp]\centering\small
\caption{Historical question-suite means and reported cell SDs. Domains are the original suite descriptions, not current manifest topic IDs. No independent-cell significance claims are retained.}
\begin{tabular}{lrr}\toprule
Suite&Same-trait CKA&Cross-trait CKA\\\midrule
200 questions, 18 domains&$.671\pm.113$&$.659\pm.089$\\
200 questions, 30 domains&$.665\pm.227$&$.621\pm.221$\\
500 questions, 50 domains&$.592\pm.271$&$.564\pm.257$\\
Original primary suite&.727&.684\\\bottomrule
\end{tabular}\end{table}
The earlier expanded suites retained different numbers of cells after invalid
values were filtered (27/108, 33/137 and 39/161 same/cross cells). Without a
matched exclusion audit these should not be treated as a controlled estimate
of the effect of domain breadth. Positive point gaps alone do not establish
robustness. The suggested optimal domain count and universal domain-coverage
recommendation are removed. Current prompt-group analyses use complete audited
metadata and explicitly report their uncertainty.

\section{Vision-encoder pilot}
\historical
The pilot used ViT-B/16, CLIP ViT-B/16 and DINOv2-base representations (width
768), five CIFAR-100 superclass contrasts and 100 image pairs per contrast.
Images across the two classes were paired by index, not by matched semantic
objects. Same-concept CKA averaged 0.560 (SD 0.134; 15 cells), versus
$-0.004$ (0.011; 30 cross cells). These overlapping observations are not an
independent-cell significance test; the earlier very small Welch p-value is
not retained as confirmatory evidence.
\begin{table}[htbp]\centering\small
\caption{Historical vision-pilot pair means and cell SDs. No functional-transfer evaluation was performed in this pilot.}
\begin{tabular}{lrrr}\toprule
Pair&Same ($n=5$)&Cross ($n=10$)&Gap\\\midrule
ViT--DINOv2&$.730\pm.043$&$-.006\pm.010$&.736\\
ViT--CLIP&$.478\pm.064$&$-.001\pm.013$&.479\\
CLIP--DINOv2&$.473\pm.054$&$-.005\pm.010$&.478\\\bottomrule
\end{tabular}\end{table}
Concept means were 0.619 (aquatic mammals versus vehicles), 0.611 (fruit and
vegetables versus insects), 0.548 (household electrical items versus trees),
0.524 (medium mammals versus household furniture), and 0.501 (people versus
flowers). Their reported SDs across three pairs were 0.122, 0.154, 0.146,
0.156 and 0.157. These observations motivate a separate matched vision study;
they do not establish cross-modal functional universality or attribute pair
differences causally to pretraining objectives.

\section{Singular values of historical affine maps}
\historical
\begin{table}[htbp]\centering\small
\caption{Historical spectral summaries of PCA-50 maps. Energy concentration is not algebraic matrix rank or a semantic dimensionality measure.}
\begin{tabular}{lrr}\toprule
Quantity&Persona maps ($n=96$)&Safety maps ($n=6$)\\\midrule
Top-5 squared-singular-value energy&$.386\pm.076$&$.417\pm.015$\\
Top-10 energy&$.575\pm.079$&$.604\pm.036$\\
Components for 90\% energy&$27.2\pm3.4$&$25.8\pm2.3$\\
Components for 80\% energy&20.2&19.0\\
Condition number&747.7&396.9\\\bottomrule
\end{tabular}\end{table}
A full-rank map can concentrate most energy in fewer directions. Thus the
earlier conclusion that these matrices are ``not full-rank'' does not follow
from a 90\% energy threshold and is withdrawn. The spectra describe particular
fitted, preprocessed maps; they do not identify the dimensionality of universally
shared concepts. The safety aggregation differs from the 96 directed persona
maps and is not a matched inferential comparison.

\section{Safety pilot}
\historical
\begin{table}[htbp]\centering\small
\caption{Historical safety-pilot geometric scores and affine accuracies. Invalid independent-PCA direct-transfer results are omitted.}
\begin{tabular}{lrr}\toprule
Pair&CKA&Affine accuracy\\\midrule
Llama--Qwen&.704&1.000\\
Llama--Gemma&.501&1.000\\
Llama--Mistral&.645&1.000\\
Qwen--Gemma&.301&1.000\\
Qwen--Mistral&.406&1.000\\
Gemma--Mistral&.646&1.000\\\bottomrule
\end{tabular}\end{table}
The submitted table gives 0.301 for Qwen--Gemma; conflicting prose gave 0.282.
We preserve the table value and flag the historical inconsistency rather than
inventing a reconciled raw value. The six-pair mean was 0.534; the expanded
six-model, 15-pair mean was reported as $0.513\pm0.146$. The original random-label
accuracy was 0.592. These exploratory results neither certify safety nor show
functional isomorphism. A safety-prompt polarity probe does not measure actual
harmful-request refusal or robustness to attacks. Revised dependency-aware
inference has not been established for this expanded pilot.

\section{Sample size and variance ratios}
\historical
At question counts 20, 50, 100, 200, 300 and 500, historical mean CKA values
were $0.638\pm0.108$, $0.620\pm0.099$, $0.724\pm0.076$,
$0.814\pm0.079$, $0.762\pm0.104$ and $0.727\pm0.097$.
The nonmonotonic sequence does not identify a bias--variance mechanism, an
optimal sample size, or systematic dilution by additional questions.

The historical difference-to-total variance ratios were 47.9\% (Llama),
39.2\% (Qwen), 68.0\% (Gemma) and 38.9\% (Mistral), with a reported mean
48.5\%. These compare variances of differently scaled objects. Differencing
can multiply signal variance by four, so a ratio can exceed 100\% without
violating probability rules. It is not the proportion of model capacity or
latent semantic variance attributable to personality. The main proposition's
signal fractions require an explicit additive decomposition; these empirical
ratios do not estimate that decomposition.

\section{Additional controls and robustness boundaries}
\historical
Handcrafted questions yielded historical CKA $0.493\pm0.139$, compared with
$0.751\pm0.005$ over five generated-question subsamples. This change is
evidence of question sensitivity, not something to dismiss because both means
are positive. The original random-style contrasts (five contrasts, 200 questions)
gave same CKA $0.450\pm0.282$, cross CKA $0.438\pm0.269$, affine accuracy
1.000 and random-label accuracy 0.495. High affine accuracy therefore does
not by itself distinguish psychologically meaningful concepts from arbitrary
prompt contrasts. The small observed geometric gap is not proof of exactly
zero population effect or a calibrated specificity guarantee.

An alternative English suite gave CKA $0.366\pm0.236$ (18 cells), compared
with the originally reported 0.727 primary mean. Original-suite subsampling
on a restricted subset gave 0.769 at 200 questions and 0.770 at the full count;
these unmatched comparisons do not assign 99.7\% of the difference causally to
question choice. Historical Chinese, Japanese and French means were
$0.593\pm0.125$ (9 cells), $0.755\pm0.077$ (18), and $0.812\pm0.056$ (18).
The Chinese comparison excluded Qwen after a float16 failure, so it is not a
matched language ranking. No typological or cross-lingual invariance claim follows.

CKA is not immune to representational non-identifiability, and a fitted affine
map need not be unique. Discarded PCA variance does not bound the loss of label
information; transferred accuracy is not a lower bound on structural alignment.
Those earlier interpretations are withdrawn. The primary synthetic controls
are sensitivity/specificity diagnostics under explicit constructions, not a
proof that the metric detects every relevant change or ignores every nuisance.

\section{Truthfulness and formality pilots}
\historical
\begin{table}[htbp]\centering\small
\caption{Historical concept CKA by pair; these are not new primary-grid replications.}
\begin{tabular}{lrr}\toprule
Pair&Truthfulness&Formality\\\midrule
Llama--Qwen&.643&.562\\
Llama--Mistral&.613&.355\\
Qwen--Mistral&.631&.450\\
Llama--Gemma&.181&.223\\
Qwen--Gemma&.222&.279\\
Gemma--Mistral&.217&.187\\\bottomrule
\end{tabular}\end{table}
The full six-pair means were $0.418\pm0.211$ for truthfulness and
$0.342\pm0.131$ for formality. Excluding Gemma gave $0.629\pm0.012$ and
$0.456\pm0.085$; such subsets are descriptive, not evidence that one
architecture is defective or that another population is universally aligned.
\begin{table}[htbp]\centering\small
\caption{Historical affine classification accuracies (\%). Source-to-target direction; no independent-PCA direct column.}
\begin{tabular}{lrr}\toprule
Direction&Truthfulness&Formality\\\midrule
Llama$\to$Qwen&100&100\\ Llama$\to$Mistral&100&100\\
Qwen$\to$Llama&100&100\\ Qwen$\to$Mistral&100&100\\
Mistral$\to$Llama&100&100\\ Mistral$\to$Qwen&100&100\\
Llama$\to$Gemma&93.4&99.5\\ Qwen$\to$Gemma&93.7&99.7\\
Gemma$\to$Llama&98.9&100\\ Gemma$\to$Qwen&99.7&100\\
Gemma$\to$Mistral&98.9&99.7\\ Mistral$\to$Gemma&92.3&100\\\bottomrule
\end{tabular}\end{table}
Within-model reference accuracy was reported as 100\%. These historical probes
classify prompt conditions, not factual correctness or reliability of generated
answers. No condition can serve as a truthfulness deployment certificate.

\section{Non-instruction concept pilots}
\historical
Code-versus-natural-language prompts contrasted writing code for a task with
explaining the task in language; reasoning-versus-recall prompts contrasted
reasoning requests with factual-answer requests. These did not use personality
system instructions. The original 50 handcrafted pairs and later 200-question
expansion are different evaluations, not interchangeable sample sizes.
\begin{table}[htbp]\centering\small
\caption{Historical 50-pair non-instruction concept CKA.}
\begin{tabular}{lrr}\toprule
Pair&Code / language&Reasoning / recall\\\midrule
Llama--Qwen&.364&.348\\ Llama--Mistral&.356&.259\\
Qwen--Mistral&.651&.455\\ Llama--Gemma&.241&.261\\
Qwen--Gemma&.220&.262\\ Gemma--Mistral&.183&.243\\\bottomrule
\end{tabular}\end{table}
Across all pairs, the historical means were $0.336\pm0.156$ and
$0.305\pm0.076$; the later expansion reported 0.300 and 0.281.
\begin{table}[htbp]\centering\small
\caption{Historical affine non-instruction transfer accuracy (\%). Within-model references were 100\%; invalid direct-coordinate results are omitted.}
\begin{tabular}{lrr}\toprule
Direction&Code / language&Reasoning / recall\\\midrule
Llama$\to$Qwen&99&99\\ Llama$\to$Mistral&99&98\\
Qwen$\to$Llama&94&98\\ Qwen$\to$Mistral&99&91\\
Mistral$\to$Llama&96&97\\ Mistral$\to$Qwen&100&98\\
Llama$\to$Gemma&99&93\\ Qwen$\to$Gemma&95&90\\
Gemma$\to$Llama&89&87\\ Gemma$\to$Qwen&97&94\\
Gemma$\to$Mistral&94&90\\ Mistral$\to$Gemma&98&91\\\bottomrule
\end{tabular}\end{table}
These are prompt-condition classification results, not evaluations of program
correctness or reasoning competence. Removing system instructions does not
remove all instruction-following, formatting or lexical effects. The pilots
motivate separate held-out studies but do not validate universal concept transfer.

\section{Cross-concept comparisons and withdrawn architecture rankings}
\historical
\begin{table}[htbp]\centering\small
\caption{Historical difference-to-total variance ratios (\%). These are not latent signal fractions; see Appendix U.}
\begin{tabular}{lrrrrr}\toprule
Model&Personality&Truthfulness&Formality&Code&Reasoning\\\midrule
Llama&47.9&20.2&17.5&77.7&40.1\\
Qwen&39.2&17.3&22.5&48.5&36.5\\
Gemma&68.0&3.8&12.2&29.2&52.3\\
Mistral&38.9&13.2&26.7&60.6&56.2\\
Mean&48.5&13.6&19.7&54.0&46.3\\\bottomrule
\end{tabular}\end{table}
The earlier cross-concept discrimination comparison included the unmatched
raw proxy described in Appendix O and an unconditional theoretical claim.
It cannot validate the old Proposition 2 or establish a precise predicted
improvement factor. The observed non-instruction cross-concept values
(approximately 0.004--0.020) are historical descriptive quantities, not a
replacement for a matched, dependency-aware multi-domain analysis.

The original encoding-depth ranking and direct-transfer prediction used
independently fitted PCA coordinates. These targets are not coordinate-invariant,
so their architecture ranking, perfect rank correlation over four models and
leave-one-out regression do not have the proposed interpretation. The four-point
deployment rule is withdrawn. No model-selection or safety policy follows from
these tables. The revised paper retains geometry and held-out recoverability
as distinct, scoped measurements.

\section{Statistical corrections, numerical versions and reproducibility}
\subsection{Trait overlap and multiplicity}
The historical mean off-diagonal trait cosines were 0.038 (Llama), 0.041
(Qwen), 0.079 (Gemma) and 0.057 (Mistral); maximum off-diagonal values were
0.596, 0.624, 0.678 and 0.580. Their average maximum is about 0.620, not an
upper bound on every model. Small mean cosine does not establish independence
or validate the assumptions of the withdrawn discrimination proposition.

The corrected two-sided Welch p-values for the original raw-positive and
contrastive-linear analyses are 0.10343 and 0.00597. Historical RBF row outputs
were not preserved; regenerated activations give 0.00106, not an exact
historical reconstruction. Across these three audited tests:
\begin{table}[htbp]\centering\small
\caption{Historical p-value reporting audit. These corrections do not fix dependence among cells; primary inference instead uses Appendix N.}
\begin{tabular}{lrrrr}\toprule
Metric&Two-sided&Holm&Bonferroni&BH\\\midrule
Raw-positive linear&.10343&.10343&.31028&.10343\\
Contrastive linear&.00597&.01194&.01791&.00896\\
Regenerated RBF&.00106&.00318&.00318&.00318\\\bottomrule
\end{tabular}\end{table}
The old count of significant architecture pairs after treating the number of
pairs as the correction family is not retained. Likewise, retrospective
``achieved power'' and predictions that adding a few models would force safety
significance are withdrawn: dependent pairs and observed effects do not justify
those prospective claims. Safety evidence remains exploratory and inconclusive
for broad geometric discrimination.

\subsection{Frozen current numerical reference}
\begin{table}[htbp]
\centering\small
\caption{Archived rebuttal outputs versus the frozen current two-thread refit. The submitted rebuttal remains unchanged. All revised primary tables use the current version.}
\begin{tabular}{lrr}\toprule
Quantity & Archived & Current \\\midrule
source/pooled\_affine & 0.881250 & 0.881042 \\
source/pooled\_orthogonal & 0.885625 & 0.885729 \\
source/pooled\_affine\_shuffled\_pairs & 0.476562 & 0.476979 \\
source/source\_classifier\_upper & 0.978125 & 0.977813 \\
template/shared\_pca & 0.574117 & 0.574874 \\
topic/shared\_pca & 0.579416 & 0.577906 \\
source/shared\_pca & 0.595938 & 0.594063 \\
\bottomrule\end{tabular}
\end{table}

The final transfer reference is the complete two-CPU-thread refit from
preserved activations, repeated twice. Both CSVs and both summary JSONs are
byte-identical across repeats. The first run used Python 3.12.13, NumPy 2.5.1,
SciPy 1.17.1, scikit-learn 1.9.0 and h5py 3.16.0, with CUDA disabled for
fitting. Both scripts used PCA dimension 50, ridge coefficient 1, 10,000
bootstrap draws and seed 2026. Bootstrap seed is distinct from checkpoint-seed
replication. Refitting did not re-extract model activations.

Against the historical archive, 73 of 20,992 question-holdout accuracies change:
34 other-trait, 23 shared-PCA, 11 shuffled, three generic, one affine and one
orthogonal row. Eleven of 576 concept-held-out rows change, including repeated
within-model references. Row keys and counts agree. Exact historical replay
at tolerance $10^{-10}$ fails; the tolerance was not relaxed.

Selected one/two/four-thread probes demonstrate numerical sensitivity. For a
shared-PCA Qwen-to-Gemma extraversion template fold, one/two threads give
0.48913043 and four give 0.44565217, matching the archive. For a Qwen-to-Mistral
neuroticism orthogonal fold, values are 0.61956522, 0.60869565 and 0.58695652;
the last matches the archive. A sampled affine discrepancy remains unresolved
by this thread sweep. One shared-PCA cell changes by 0.32954545 between archive
and current run. For pooled openness, Qwen-to-Llama affine/orthogonal values
are 0.99/0.92 with two threads and 1.00/0.91 with one/four. These bounded
diagnostics do not recover every historical runtime setting or explain every
changed row. No diagnostic cells replace the complete two-thread reference.

\subsection{Reproduction map and preservation boundaries}
The preserved package contains the ordered prompt manifest, 64 HDF5 activation
files with four layers each, extraction and analysis code, 3072 primary CKA
cells, grouped CKA grids, fold-level transfer outputs, 300 synthetic repetition
outputs, and numerical audit records. Every cached layer matrix was checked
for finiteness and metadata consistency. Replayed CKA summaries agree within
$1.11\times10^{-16}$; synthetic summary fields agree exactly. These checks
are separate from the failed historical transfer row-exact comparison.

To reproduce from cached activations, run the grouped CKA replay, the
same-trait transfer script, and the concept-held-out script with the fixed
parameters above. The scripts retain exact cell/fold outputs; generated tables
read the frozen current summaries. The two transfer scripts are
\path{35_rebuttal_heldout_transfer.py} and
\path{38_rebuttal_concept_heldout_transfer.py}; the reporting audit is
\path{39_rebuttal_pvalue_audit.py}. A full activation re-extraction
additionally requires the original model weights, chat templates, numerical
precision and extraction configuration. Model licenses still apply.

The original supplement's older environment and compute-cost statements
describe its historical runs, not the present refit. No new end-to-end GPU
efficiency or cost advantage is inferred here. Original PDFs and archived
rebuttal outputs remain unchanged. Editable supplementary source was
reconstructed from the submitted PDF; it is not the recovered original TeX
project. Historical pilots above have not all been re-executed, and public
artifact access is not asserted merely because the author has a private backup.

\section*{Responsible interpretation}
Prompted traits are experimental labels, not assessments of human personality.
Manipulating sycophancy, confidence or helpfulness can have harmful applications;
recoverability does not establish truthfulness, user benefit or safety. These
results should not guide psychological profiling or replace behavioral safety
evaluation. No deployment guarantee or claim that benefits outweigh all risks
is made. Model weights and any eventual shared artifacts must respect their
licenses and applicable data-release constraints.